\documentclass{article}

\usepackage[main,final,nonatbib]{neurips_2026}

\usepackage[utf8]{inputenc}
\usepackage[T1]{fontenc}
\usepackage{hyperref}
\usepackage{url}
\usepackage{booktabs}
\usepackage{amsfonts}
\usepackage{nicefrac}
\usepackage{microtype}

\usepackage{amsmath,amssymb,mathtools,bm}
\usepackage{amsthm}
\usepackage[table]{xcolor}
\usepackage[backend=biber,style=authoryear,natbib=true,maxcitenames=2,maxbibnames=20,uniquelist=false]{biblatex}
\usepackage{xurl}

\usepackage{algorithm}
\usepackage[noend]{algpseudocode}
\usepackage{tikz}
\usepackage[nameinlink,capitalize,noabbrev]{cleveref}

\hypersetup{
  colorlinks=true,
  linkcolor=black,
  citecolor=blue!55!black,
  urlcolor=blue!55!black,
  pdfauthor={Aritra Das, Dhruman Gupta, Debayan Gupta},
  pdftitle={Linearized 2-Simplicial Attention}
}
\definecolor{oursrow}{RGB}{238,242,248}
\definecolor{acc}{RGB}{65,91,145}

\newtheorem{proposition}{Proposition}
\newtheorem{lemma}{Lemma}

\newcommand{\R}{\mathbb{R}}
\newcommand{\E}{\mathbb{E}}

\newcommand{\rmsnorm}{\operatorname{RMSNorm}}
\newcommand{\concat}{\operatorname{Concat}}
\newcommand{\had}{\odot}
\newcommand{\Simp}{\textsc{LinSimp}}

\newcommand{\best}[1]{\textbf{#1}}

\title{Linearized 2-Simplicial Attention}

\author{%
  Aritra Das \\
  Truth Audit Labs
  \And
  Dhruman Gupta \\
  Ashoka University \\
  Truth Audit Labs
  \And
  Debayan Gupta \\
  Ashoka University \\
  Truth Audit Labs
}
\date{}

\begin{document}
\maketitle

\begin{abstract}
We present a linearized form of 2-simplicial attention by rewriting the trilinear score as an inner product between a composite query and a key, so that the sum over keys takes the same form as ordinary softmax attention. We then approximate this sum with positive random features and store all previous keys in a fixed-size state, while the sum over second keys stays exact over all previous tokens. This reduces the cost from cubic to quadratic in sequence length, and both keys range over all previous tokens, which windowed 2-simplicial attention does not allow. We implement it with custom Triton kernels and combine it with Kimi Delta Attention to build a model with no softmax attention at all. With 15B equal training tokens on FineWeb-Edu, this model lowers validation loss from 2.35 to 2.32 over a KDA hybrid and lowers LAMBADA perplexity from 40 to 38.
\end{abstract}

\section{Introduction}

Softmax Transformers are widely used in sequence modeling tasks such as causal language modeling~\citep{vaswani2017attention} due to their high performance and parallelized training. However, their quadratic computational cost is often limiting. Optimized kernels make softmax attention faster in practice \citep{dao2022flashattention,dao2024flashattention2}. However, they do not change its $\mathcal{O}(n^2)$ scaling. Linear attention and state-space models replace the attention matrix with a fixed-size recurrent state. This turns sequence mixing into a linear-time recurrence whose memory does not grow with context length \citep{katharopoulos2020transformers,sun2023retnet,gu2023mamba,poli2023hyena}. Unfortunately, a fixed state can make exact retrieval and selective updates difficult. Recent delta-rule models improve this trade-off by correcting, rather than only adding to, an associative memory \citep{yang2024deltanet,yang2024gateddelta,kimi2025linear}.

Standard attention models pairwise relations between tokens. However, reasoning steps such as composing two relations depend on three tokens at once. The 2-simplicial Transformer introduces such higher-order interactions through a trilinear score and a value that combines two earlier tokens \citep{clift2020logic}. Related triangular mechanisms have been used for systematic generalization \citep{bergen2021edge}, and recent theory identifies triple-dependent tasks that are hard for a single standard attention layer \citep{sanford2023representational,kozachinskiy2025strassen}. However, this expressivity comes at a cost. Each query attends to all pairs of earlier tokens, which leads to $\mathcal{O}(n^3)$ computation. Recent work makes this practical by restricting both axes to fixed windows and using a custom Triton kernel \citep{roy2025fast}. This leads us to ask whether techniques from linear attention can retain the expressivity of 2-simplicial attention at the $\mathcal{O}(n^2)$ cost of standard attention.

We propose \Simp. It uses positive random features \citep{rahimi2007random,choromanski2021performer} to summarize all previous keys in a fixed-size state, and sums over the second keys exactly. This gives $\mathcal{O}(n^2)$ cost. We replace the six softmax-attention layers of a KDA hybrid \citep{kimi2025linear} with \Simp{}. With 15B equal training tokens on FineWeb-Edu, this lowers validation loss from 2.35 to 2.32 and LAMBADA perplexity from 40 to 38. Our contributions are:

\begin{itemize}
    \item An approximation of 2-simplicial attention with $\mathcal{O}(n^2)$ cost
    \item Custom Triton forward and backward kernels
    \item An iso-token comparison with a KDA hybrid.
\end{itemize}

\section{Background}

Following \citet{clift2020logic} and \citet{roy2025fast}, let $q_t,k_t,r_t\in\R^D$ and $v_t,u_t\in\R^{D_v}$ be the query, key, second key, value, and second value of one attention head at position $t$. Let $\hat q_t=q_t/\lVert q_t\rVert_2$, and define $\hat k_t$ and $\hat r_t$ in the same way. With a learned per-head temperature $\tau$, the trilinear score is
\begin{equation}
  s_{ijc}
  = \langle \tau\hat q_i,\hat k_j,\hat r_c\rangle
  := \sum_{a=1}^{D}\tau\hat q_{ia}\hat k_{ja}\hat r_{ca}.
  \label{eq:trilinear-score}
\end{equation}
The exact causal output uses one softmax over all pairs $(j,c)$ with $j,c\le i$:
\begin{equation}
  o_i
  =
  \frac{
    \sum_{j\le i}\sum_{c\le i}
    \exp(s_{ijc})\,(v_j\had u_c)
  }{
    \sum_{j\le i}\sum_{c\le i}\exp(s_{ijc})
  }.
  \label{eq:exact-full}
\end{equation}
Each query has $O(T^2)$ pairs, so the total cost is $O(T^3)$. Grouping the query with the second key turns the score into a dot product with the key,
\begin{equation}
  s_{ijc}
  = \bigl((\tau\hat q_i)\had\hat r_c\bigr)^\top\hat k_j,
  \label{eq:two-views}
\end{equation}
which has the same form as the score in standard attention.

\section{Method}
We use the main idea of linear attention, which replaces the exponential score with a dot product of feature maps so that the sum over keys can be stored as a running state \citep{katharopoulos2020transformers}. We apply this to one of the two keys and compute the sum over the other key exactly. This works because, for a fixed second key, the score in \cref{eq:two-views} is a dot product between the key and a composite query.

\label{sec:method}

\subsection{Positive random features}

For fixed $(i,c)$, define the composite query
\begin{equation}
  z_{ic}=(\tau\hat q_i)\had\hat r_c.
\end{equation}
Then $\exp(s_{ijc})=\exp(z_{ic}^{\top}\hat k_j)$. Let $\Omega\in\R^{m\times D}$ have rows with standard Gaussian marginals and define
\begin{equation}
  \phi(x)
  =\frac{1}{\sqrt m}
    \exp\!\left(\Omega x-\frac{\lVert x\rVert_2^2}{2}\mathbf{1}\right)
  \in\R_+^m,
  \label{eq:prf}
\end{equation}
where the exponential is elementwise.

\begin{lemma}
\label{lem:kernel}
For the feature map in \cref{eq:prf},
$\E_{\Omega}\bigl[\phi(x)^\top\phi(y)\bigr]=\exp(x^\top y)$.
\end{lemma}
\begin{proof}
For a standard Gaussian row $\omega$,
$\E[\exp(\omega^\top(x+y))]=\exp(\lVert x+y\rVert_2^2/2)$. Multiplying by the two norm-correction terms gives $\exp(x^\top y)$. Averaging over $m$ rows preserves the expectation.
\end{proof}

We use orthogonal blocks with chi-distributed row norms \citep{yu2016orthogonal,choromanski2021performer}. Positive features keep the denominator nonnegative.

Substituting $\exp(z_{ic}^{\top}\hat k_j)\approx\phi(z_{ic})^\top\phi(\hat k_j)$ into \cref{eq:exact-full}, the sum over $j$ collects into two prefix states:
\begin{align}
  M_i &= \sum_{j\le i}\phi(\hat k_j)v_j^\top
       \in\R^{m\times D_v},
  & M_i&=M_{i-1}+\phi(\hat k_i)v_i^\top,
  \label{eq:Mstate}\\
  a_i &= \sum_{j\le i}\phi(\hat k_j)
       \in\R^m,
  & a_i&=a_{i-1}+\phi(\hat k_i).
  \label{eq:astate}
\end{align}
\subsection{Summing over all second keys}

We sum the second-key index $c$ over all previous tokens:
\begin{equation}
  N_i=\sum_{c=1}^{i}\bigl[\phi(z_{ic})^\top M_i\bigr]\had u_c,
  \qquad
  Z_i=\sum_{c=1}^{i}\phi(z_{ic})^\top a_i,
  \qquad
  \widetilde o_i=\frac{N_i}{Z_i+\varepsilon},
  \label{eq:Ni}
\end{equation}
with $\varepsilon=10^{-30}$. Expanding the states gives
\begin{equation}
  \widetilde o_i
  =\frac{\sum_{c\le i}\sum_{j\le i}\widetilde\kappa_{ijc}\,(v_j\had u_c)}
        {\sum_{c\le i}\sum_{j\le i}\widetilde\kappa_{ijc}+\varepsilon},
  \qquad
  \widetilde\kappa_{ijc}=\phi(z_{ic})^\top\phi(\hat k_j)\ge 0,
  \label{eq:expanded}
\end{equation}
which is a random-feature approximation of \cref{eq:exact-full} over the same set of pairs. The state is $M_i$ for every second key; using $M_c$ would restrict the sum to $j\le c$. Numerator and denominator share one normalizer over all pairs $(j,c)$. The $j$ axis is stored in a rank-$m$ state, and the $c$ axis is computed explicitly.

\begin{proposition}
\label{prop:rank}
If the pair weights factor as $w_{ijc}=\alpha_{ij}\beta_{ic}$, the jointly normalized output is
\begin{equation}
  \frac{\sum_{j,c\le i}\alpha_{ij}\beta_{ic}\,(v_j\had u_c)}{\sum_{j,c\le i}\alpha_{ij}\beta_{ic}}
  =\frac{\sum_{j\le i}\alpha_{ij}v_j}{\sum_{j\le i}\alpha_{ij}}
   \had
   \frac{\sum_{c\le i}\beta_{ic}u_c}{\sum_{c\le i}\beta_{ic}},
\end{equation}
the elementwise product of two ordinary attention outputs. For \Simp, the matrix $[\widetilde\kappa_{ijc}]_{c,j\le i}=\Phi_i\Psi_i^\top$, where $\Phi_i$ has rows $\phi(z_{ic})^\top$ and $\Psi_i$ has rows $\phi(\hat k_j)^\top$, has rank at most $m$.
\end{proposition}
Both claims follow by factoring the sums. Thus, factored weights choose $j$ and $c$ independently, while in \Simp{} the weight on $j$ depends on $c$ through $m$ features.

\subsection{Output gate}

Let $\bar x_i\in\R^{d_{\mathrm{model}}}$ be the normalized layer input and let $\widetilde o_i^h$ be the output of head $h$. With $W_{g,\downarrow}\in\R^{\rho\times d_{\mathrm{model}}}$ and $W_{g,\uparrow}^{h}\in\R^{D_v\times\rho}$,
\begin{equation}
  g_i^h=\sigma\!\left(W_{g,\uparrow}^{h}W_{g,\downarrow}\bar x_i+b_g^h\right)\in(0,1)^{D_v},
  \qquad
  y_i=W_o\,\concat_{h=1}^{H}\left[g_i^h\had\rmsnorm(\widetilde o_i^h)\right].
  \label{eq:gated-output}
\end{equation}
The down-projection is shared across heads. A fixed scalar gate per head can be absorbed into $W_o$; $g_i^h$ depends on the token and acts on each channel. This follows the output gate of Kimi Linear \citep{kimi2025linear}.

\subsection{Complexity}

\begin{proposition}
\label{prop:complexity}
For fixed $m$, $D$, and $D_v$, the layer in \cref{eq:Mstate,eq:astate,eq:Ni} is causal and has $O(T^2)$ total mixing work over a sequence of length $T$. At decoding position $i$, it stores $M_i$, $a_i$, and $\{(\hat r_c,u_c)\}_{c\le i}$, which is $O(mD_v+m+i(D+D_v))$ state per head, and its mixing work is $O(im(D+D_v))$.
\end{proposition}
\begin{proof}
Every index in \cref{eq:expanded} satisfies $j,c\le i$, so the output at $i$ does not depend on later positions. Updating $M_i$ and $a_i$ costs $O(m(D+D_v))$. Each of the $i$ second keys requires $\phi(z_{ic})$ and its contraction with $M_i$ and $a_i$, which costs $O(m(D+D_v))$. Summing over $i\le T$ gives $O(T^2m(D+D_v))$. The gate adds work linear in $T$.
\end{proof}

Per $(i,c)$ pair, \Simp{} uses about $2m(D+D_v)$ FLOPs. Softmax attention with head dimension $d_h$ uses about $4d_h$ FLOPs per query--key pair.

\section{Model and Implementation}
\label{sec:architecture}

All models use 24 pre-norm residual blocks with $d_{\mathrm{model}}=1024$, each with a token mixer and a SwiGLU MLP \citep{shazeer2020glu}, RMSNorm \citep{zhang2019rmsnorm}, tied input and output embeddings, and the Llama~2 tokenizer with a 32k vocabulary \citep{touvron2023llama2}. KDA is a gated delta-rule layer with channel-wise decay \citep{kimi2025linear}; we use six 128-dimensional heads, value expansion 1, and a depthwise short convolution of width 4, through the FLA library \citep{yang2024fla}. Each \Simp{} layer has six heads with $D=D_v=64$, five input projections $(q,k,r,v,u)$, and a fixed random-feature bank of rank $m=128$ per head and layer. We initialize $\tau=\sqrt D$ and $b_g^h=-2$, and use a $5\times$ learning-rate multiplier for $\tau$. We implement \Simp{} with Triton forward and backward kernels \citep{tillet2019triton} (\cref{app:kernels}).

The KDA hybrid uses KDA in 18 layers and softmax attention with rotary embeddings \citep{su2021roformer} in layers $\{3,7,11,15,19,23\}$, following the 3:1 pattern of Kimi Linear. The \Simp{} model replaces these six attention layers with \Simp{}. We adjust only the SwiGLU width so that both models have the same non-embedding parameter count to within $0.5\%$.

We train both models on FineWeb-Edu \citep{penedo2024fineweb} for 15B tokens at 2k context. Both runs use the same tokens, data order, random seed, bf16 precision, peak learning rate $3\times10^{-4}$, and warmup--stable--decay schedule \citep{wen2024wsd}. We report validation loss on held-out FineWeb-Edu, LAMBADA perplexity and accuracy \citep{paperno2016lambada}, and WikiText word perplexity \citep{merity2016wikitext}, using the Language Model Evaluation Harness \citep{gao2023lmeval}.

\begin{table}[h]
\centering
\caption{FineWeb-Edu, 15B training tokens, equal-token comparison.}
\label{tab:fineweb}
\begin{tabular}{lrrrr}
\toprule
Model & Val.\ loss $\downarrow$ & LAMBADA ppl.\ $\downarrow$ & LAMBADA acc.\ $\uparrow$ & WikiText ppl.\ $\downarrow$ \\
\midrule
KDA hybrid & 2.35 & 40 & 32\% & 30 \\
\rowcolor{oursrow}
KDA + \Simp{} & \best{2.32} & \best{38} & \best{33\%} & \best{28.5} \\
\bottomrule
\end{tabular}
\end{table}

\Cref{tab:fineweb} gives the results. The \Simp{} model is better on all four metrics.

\section{Limitations}
\label{sec:limitations}

Results use a single seed and one model size. Per pair, \Simp{} uses more FLOPs than softmax attention (\cref{prop:complexity}), so prefill is costly at long context. 

\clearpage
\printbibliography

\appendix

\section{Related Work}

\paragraph{Efficient sequence models.}
Linear attention computes the token sum as a recurrent state \citep{katharopoulos2020transformers}. Performer approximates the softmax kernel with positive orthogonal random features \citep{choromanski2021performer}, building on random features \citep{rahimi2007random} and orthogonal feature banks \citep{yu2016orthogonal}. GLA adds data-dependent gates and chunked training \citep{yang2024gla}. DeltaNet uses an error-correcting delta update \citep{yang2024deltanet}, Gated DeltaNet adds forgetting \citep{yang2024gateddelta}, and KDA adds channel-wise decay within a hybrid that keeps periodic softmax-attention layers \citep{kimi2025linear}. State-space models, retention, and long convolutions are other subquadratic designs \citep{sun2023retnet,gu2023mamba,poli2023hyena}. We use a linear-attention state for one token axis of a trilinear interaction and place the layer in the attention slots of a KDA hybrid.

\paragraph{Higher-order attention.}
The 2-simplicial Transformer scores triples of tokens and combines two value streams \citep{clift2020logic}. Edge Transformer uses triangular attention over pair states \citep{bergen2021edge}. Triple-detection tasks separate shallow standard attention from higher-order mechanisms \citep{sanford2023representational}, and Strassen attention gives a subcubic construction \citep{kozachinskiy2025strassen}. \citet{roy2025fast} compute exact 2-simplicial attention over a $w_1\times w_2$ window with a Triton kernel. We use the same trilinear score and pair values over all previous tokens on both axes.

\paragraph{GPU kernels.}
FlashAttention uses tiling and recomputation for exact attention \citep{dao2022flashattention,dao2024flashattention2}. Online softmax keeps a running maximum and normalizer \citep{milakov2018online}. Our kernel uses the same ideas in Triton \citep{tillet2019triton}, with an additional second-key axis and feature axis.

\section{Derivation}
\label{app:derivation}

Let $y_{jc}=v_j\had u_c$ and $\kappa_{ijc}=\exp(z_{ic}^\top\hat k_j)$. With the feature approximation,
\begin{align}
  \sum_{c\le i}\sum_{j\le i}\kappa_{ijc}y_{jc}
  &\approx
  \sum_{c\le i}\sum_{j\le i}
  \bigl[\phi(z_{ic})^\top\phi(\hat k_j)\bigr](v_j\had u_c)\\
  &=\sum_{c\le i}
  \Bigl(\phi(z_{ic})^\top\sum_{j\le i}\phi(\hat k_j)v_j^\top\Bigr)\had u_c
  =\sum_{c\le i}\bigl[\phi(z_{ic})^\top M_i\bigr]\had u_c.
\end{align}
Setting the value to one gives the denominator,
$\sum_{c\le i}\sum_{j\le i}\kappa_{ijc}\approx\sum_{c\le i}\phi(z_{ic})^\top a_i$. The other orientation in \cref{eq:two-views} is also valid; using both doubles the cost.

The factor $m^{-1/2}$ in each feature vector gives $1/m$ in every kernel estimate. It appears in both $N_i$ and $Z_i$ and can be omitted if both use the same convention.

\section{Random-Feature Bank and Stabilization}
\label{app:features}

\paragraph{Feature bank.}
For each head and layer, the bank is sampled once and stored as a fixed buffer:
\begin{enumerate}
  \item Draw $G\in\R^{D\times D}$ with iid standard Gaussian entries and compute $G=QR$.
  \item Multiply each column of $Q$ by $\operatorname{sign}(R_{aa})$, so that $Q$ is Haar-distributed.
  \item Scale each row of $Q^\top$ by an independent $\chi_D$ radius.
  \item Stack blocks until there are $m$ rows.
\end{enumerate}
Each row has marginal $\mathcal N(0,I_D)$, so \cref{lem:kernel} holds.

\paragraph{Log-space evaluation.}
For query $i$, second key $c$, and feature row $p$, let
\begin{equation}
  \ell_{icp}=\omega_p^\top z_{ic}-\frac{\lVert z_{ic}\rVert_2^2}{2},
  \qquad
  L_i=\max_{c\le i,\,p\le m}\ell_{icp}.
\end{equation}
The kernel computes $L_i$ online over second-key tiles and feature blocks and evaluates $\exp(\ell_{icp}-L_i)$. Every pair weight at query $i$ is multiplied by $\exp(-L_i)$, which cancels in $N_i/Z_i$. $L_i$ depends only on $c\le i$, so causality is preserved. Key features are evaluated without a shift; normalized keys and fixed row norms keep their log values bounded.

\section{Triton Kernel}
\label{app:kernels}

The host computes the key features $\phi(\hat k_j)$ in fp32, splits the sequence into 32-token chunks, and forms exclusive chunk-boundary states for $M$ and $a$. One forward program runs per query row and batch--head pair. It rebuilds $M_i$ and $a_i$ from the chunk boundary and a masked tail, then loops over the second keys $c\le i$ in fixed-size tiles and over the feature rank in blocks of at most 128. For each tile and block it computes the log-features of $z_{ic}$, updates one running maximum shared by all tiles and blocks, rescales the running $N$ and $Z$, and accumulates in fp32. The shift multiplies $N_i$ and $Z_i$ by the same factor and cancels in their ratio (\cref{app:features}).

The backward pass saves the inputs and one stabilizer per query and head, and recomputes the features. The gradients $dr_c$ and $du_c$ are sums over all queries $i\ge c$. No tensor of shape $T\times T\times m$ is stored. \Cref{alg:forward} gives the forward pass.

\begin{algorithm}[h]
\caption{Forward pass for query $i$ and one batch--head pair}
\label{alg:forward}
\begin{algorithmic}[1]
\Require $q_i$; second keys $(r_c,u_c)_{c\le i}$; key features $\phi(\hat k_j)$; values $v_j$; boundary states $M_b,a_b$; bank $\Omega$
\State Rebuild $M_i,a_i$ from the chunk boundary and the masked tail $j\le i$
\State $L\gets-\infty$, $N\gets0$, $Z\gets0$
\For{second-key tiles $C\subseteq\{1,\ldots,i\}$}
  \State $z_{ic}\gets(\tau\hat q_i)\had\hat r_c$ for $c\in C$
  \For{feature blocks $P\subseteq\{1,\ldots,m\}$}
    \State $\ell_{cP}\gets z_{ic}\Omega_P^\top-\lVert z_{ic}\rVert_2^2/2$
    \State $L'\gets\max(L,\max_{c\in C,p\in P}\ell_{icp})$
    \State $N\gets e^{L-L'}N$, $Z\gets e^{L-L'}Z$
    \State $\psi_{cP}\gets\exp(\ell_{cP}-L')$
    \State $N\gets N+\sum_{c\in C}(\psi_{cP}^\top M_{i,P})\had u_c$
    \State $Z\gets Z+\sum_{c\in C}\psi_{cP}^\top a_{i,P}$
    \State $L\gets L'$
  \EndFor
\EndFor
\State Save $L$ and return $N/(Z+10^{-30})$
\end{algorithmic}
\end{algorithm}

\clearpage

\end{document}